\documentclass[11pt]{article}

\usepackage[margin=1in]{geometry}
\usepackage{amsmath,amssymb}
\usepackage{microtype}
\usepackage{graphicx}
\usepackage{subcaption}
\usepackage{booktabs}
\usepackage{algorithm}
\usepackage{algorithmic}
\usepackage{hyperref}
\usepackage{mathtools}
\usepackage{amsthm}
\usepackage[numbers,sort&compress]{natbib}
\usepackage[capitalize,noabbrev]{cleveref}

\theoremstyle{plain}
\newtheorem{theorem}{Theorem}[section]
\newtheorem{proposition}[theorem]{Proposition}

\theoremstyle{definition}

\theoremstyle{remark}

\title{\textbf{A Plug-in Interpretation of Conditioning in\\
Score-Based Diffusion Models}}

\author{
Libo Chen$^{1}$ \quad
Souvik Ghosh$^{2}$ \quad
Teo Deveney$^{1}$ \quad
Chris Budd$^{1}$ \quad
Vinay P. Namboodiri$^{1}$
\\[0.7em]
\small $^{1}$University of Bath, Bath, UK
\\
\small $^{2}$International Institute of Information Technology Hyderabad, India
}

\date{}

\begin{document}

\maketitle

\begin{abstract}
We propose a conditioning mechanism for diffusion models based on multi-speed joint diffusion of the target and the condition.
The mechanism learns an unconditional joint score network and enforces conditioning at inference via a plug-in correction term.
The plug-in term separates the conditioning contribution from the learned unconditional dynamics, offering a transparent view of how the condition steers generation of the target distribution.
Building on this, we derive explicit conditional reverse-time SDEs and approximate probability-flow ODEs, enabling principled and directly comparable conditional samplers.
To reduce the induced ODE--SDE discrepancy, we introduce a log-Fokker--Planck residual regularization that improves ODE sampling quality.
Experiments on conditional image generation tasks demonstrate competitive performance and support the effectiveness of the plug-in conditioning view.
Additional ODE--SDE comparison experiments show that the log-Fokker--Planck residual regularization improves deterministic ODE sampling.
\end{abstract}


\section{Introduction}
\label{sec:intro}

Score-based diffusion models~\cite{song2020score} have emerged as a powerful class of generative models, achieving state-of-the-art performance in image synthesis~\cite{dhariwal2021diffusion}, audio generation~\cite{kong2020diffwave}, and inverse problems~\cite{song2021solving}. In many practical settings, however, the goal is conditional generation: producing samples that are consistent with observations or constraints, such as arbitrary masks in image inpainting~\cite{lugmayr2022repaint}, low-resolution measurements in super-resolution~\cite{saharia2022image}, or noisy partial observations in data assimilation~\cite{rozet2023score}.

Conditional generation requires sampling from a condition-dependent reverse process. This, in turn, depends on estimating a conditional score function, which is generally intractable and must therefore be approximated. Existing approaches primarily differ in how the conditioning information is incorporated into the diffusion dynamics.

Following~\cite{batzolis2021conditional}, existing methods can broadly be divided into two families based on the treatment of the condition during the forward diffusion process. Fixed-condition approaches diffuse only the target state while keeping the condition unchanged, whereas joint-diffusion approaches diffuse both the target and the condition and learn a joint score field.

\textbf{Fixed-condition estimators.}
The first family keeps the condition fixed and learns a conditional score network directly from paired data via denoising score matching, often by concatenating the condition into the score model~\cite{song2019generative,vincent2011connection}. While effective, this design tightly couples the learned model to a specific conditioning interface and observation operator. As a result, adapting to new inverse problems may require substantial re-tuning or retraining. Moreover, because the condition itself is not modeled as a stochastic variable within the diffusion process, these methods provide limited insight into how observations explicitly influence the reverse-time dynamics.

\textbf{Diffused-condition estimators.}
The second family jointly diffuses the target and the condition and learns a joint score field. This includes the conditional diffusive estimator (CDiffE) and its multi-speed variant (CMDE)~\cite{batzolis2021conditional}, where the condition may evolve under a different noise schedule to balance optimization ease and approximation fidelity. Although more flexible, existing joint-diffusion approaches typically rely on implicit schedule matching and do not provide an explicit conditional reverse-time SDE/ODE formulation. Consequently, the mechanism through which observations guide the reverse dynamics remains less transparent.

More recently, pretrained-prior approaches have emerged as an effective
alternative for inverse problems. Methods such as Diffusion Posterior Sampling
(DPS)~\cite{chung2022diffusion} and DiffPIR~\cite{zhu2023denoising} combine
pretrained diffusion models with observation-consistency updates during
sampling. While DPS relies on iterative likelihood-gradient evaluations
through the pretrained score network, which can increase computational cost. DiffPIR introduces an
optimization-inspired data-consistency step.
Our work treats
conditioning as an analytic correction to a learned score, applies both to pretrained-prior restoration and to joint target--condition diffusion, and leads directly to conditional reverse SDE/ODE dynamics.

In this work, we propose a unified conditional diffusion framework based on multi-speed joint diffusion with analytic plug-in conditioning. We jointly diffuse the target and the condition under potentially different noise schedules and train a single unconditional joint score model. Conditioning is imposed explicitly at inference time through an analytic plug-in correction derived from the forward corruption process, rather than being absorbed implicitly into the learned dynamics. This yields a modular conditional score estimator that separates unconditional generation dynamics from observation-consistency corrections.

Building on this, we derive explicit conditional reverse-time stochastic differential equations (SDEs) together with matched approximate probability-flow ordinary differential equations (ODEs). The conditioning contribution appears in closed form within the reverse dynamics, providing a transparent interpretation of how observations influence both stochastic and deterministic sampling trajectories.

Beyond methodology, inspired by recent progress on reducing the ODE--SDE discrepancy in unconditional diffusion models~\cite{deveney2025closing}, we analyze the discrepancy between conditional ODE and SDE samplers. We show that this discrepancy can be bounded by a log-Fokker--Planck residual associated with the learned score dynamics. This provides a mechanism for improving deterministic ODE sampling quality: as the residual decreases, the probability-flow ODE more closely tracks the reverse-time SDE distribution.

Empirically, we evaluate the proposed framework on conditional image generation tasks including inpainting and super-resolution across CelebA, FFHQ, and ImageNet. We compare against representative conditional diffusion approaches as well as modern pretrained-prior methods, including DPS~\cite{chung2022diffusion} and DiffPIR~\cite{zhu2023denoising}. Our method achieves competitive or superior reconstruction fidelity and conditional consistency while avoiding the per-step backpropagation through the diffusion model required by gradient-based posterior sampling approaches.

\paragraph{Contributions.}
In summary, our main contributions are:
\begin{itemize}
\item We provide another interpretation of how conditional diffusion relates to the sampling time via an analytical plug-in term. We also show that our formulation can work with a pre-trained model for inpainting (Table~\ref{tab:pretrained-prior-inpainting}).
\item We derive explicit conditional reverse-time SDE and matched approximate probability-flow ODE formulations with interpretable conditioning dynamics.
\item We establish a theoretical connection between the conditional ODE--SDE discrepancy and a controllable log-Fokker--Planck residual.

\end{itemize}

\section{Background: Conditional and Posterior-Guided Diffusion}

Many applications require sampling from a conditional target distribution $p(x_0|y)$, where $y$ denotes an observation, measurement, or side information. In score-based diffusion models, conditional generation is achieved by approximating the conditional score field that governs the reverse-time dynamics. Existing approaches can broadly be categorized into two families: methods that learn conditional score models directly from paired data, and methods that impose observations on a pretrained diffusion prior during sampling.

\paragraph{Conditional score estimation.}
A common approach is to learn a score model that directly approximates the conditional score $\nabla_x \log p_t(x_t|y)$. In fixed-condition estimators, the condition is kept unchanged throughout the diffusion process and provided directly to the score network. Representative examples include conditional denoising estimators (CDE), where the score model is trained on paired data and conditioned explicitly on the observation.

An alternative is to jointly diffuse both the target and the condition and learn a joint score field. This includes the conditional diffusive estimator (CDiffE) and its multi-speed extension (CMDE)~\cite{batzolis2021conditional}. In these methods, the condition evolves along the diffusion process and conditional generation is recovered by evaluating the learned joint score using a perturbed version of the observation. While joint-diffusion approaches offer greater flexibility, the role of the observation in the reverse dynamics remains largely implicit.

\paragraph{Posterior-guided diffusion sampling.}
A second family of methods uses a pretrained unconditional diffusion model as a prior and imposes the observation only at sampling time. These approaches are particularly attractive for inverse problems because they avoid retraining the diffusion model for each observation operator.

Diffusion Posterior Sampling (DPS)~\cite{chung2022diffusion} follows a Bayesian formulation in which the posterior score is decomposed as

\begin{equation}
\nabla_x \log p_t(x_t|y)
=
\nabla_x \log p_t(x_t)
+
\nabla_x \log p_t(y|x_t).
\label{eq:dps_decomposition}
\end{equation}

The first term is provided by the pretrained diffusion model, while the second is approximated through likelihood-gradient guidance during sampling. DPS has demonstrated strong performance on a wide range of noisy inverse problems, but requires differentiating the likelihood term at every sampling step.

DiffPIR~\cite{zhu2023denoising} adopts a related plug-and-play perspective and combines pretrained diffusion priors with iterative data-consistency updates inspired by optimization-based image restoration. Rather than directly estimating a conditional score, DiffPIR alternates between diffusion-based denoising and observation-consistency corrections.

\paragraph{Our perspective.}
The methods above incorporate observations either implicitly through learned conditional score models or through iterative posterior-guidance procedures applied to pretrained priors. In contrast, we derive an analytic plug-in correction that explicitly separates unconditional generation dynamics from observation-consistency terms. This yields a unified conditioning mechanism that applies both to jointly trained diffusion models and to pretrained diffusion priors, while avoiding the repeated likelihood-gradient evaluations required by posterior-sampling approaches such as DPS.

\section{Our method}

\subsection{Plug-in conditioning with a joint score (Contribution 1)}
\label{subsec:unified-conditioning}
We first introduce a conditioning mechanism that avoids learning a separate
conditional score model for each $Y_0$.
Instead of approximating $\nabla_{x} \log p_t(X_t\mid Y_0)$ directly, we jointly
diffuse $(X_t,Y_t)$ and inject conditioning through an explicit plug-in correction in the $y$-component.

We diffuse the forward pair $(X_t,Y_t)$ with possibly different diffusion intensities:
\begin{equation}
\begin{aligned}
  dX_t &= f_X(X_t,t)\,dt + g_X(t)\,dW^1_t,\\
  dY_t &= f_Y(Y_t,t)\,dt + g_Y(t)\,dW^2_t,
\end{aligned}
\label{eq:joint-forward}
\end{equation}
where $W^1$ and $W^2$ are independent Brownian motions. Let $Z_t=(X_t,Y_t)$,
define $f(Z_t,t)=(f_X,f_Y)$, and set
$G(t)=\mathrm{diag}(g_X(t)I,\,g_Y(t)I)$.

For VE-SDE type forward processes (with zero drift), the marginal noising admits a linear-Gaussian form:
\begin{equation}
\begin{aligned}
X_t &= X_0 + \sigma_X(t)\,z_1,\\
Y_t &= Y_0 + \sigma_Y(t)\,z_2,
\end{aligned}
\end{equation}
where $z_1,z_2\sim\mathcal{N}(0,I)$ are mutually independent and independent of $(X_0,Y_0)$, and
$\sigma_X^2(t)=\int_0^t g_X^2(s)\,ds$, $\sigma_Y^2(t)=\int_0^t g_Y^2(s)\,ds$.
Under independent Gaussian assumptions, we have
$X_t \perp Y_t \mid Y_0$ (Supplementary \textbf{Section 1} for more details), hence
\begin{equation}
  p_t(X_t,Y_t\mid Y_0)
  =
  p_t(X_t\mid Y_0)\,p_t(Y_t\mid Y_0).
  \label{eq:cond-indep}
\end{equation}
Therefore, in conditional generation we simulate the reverse-time dynamics of $(X_t,Y_t)$ conditioned on $Y_0$
and finally return the \(x\)-component at \(t=0\) as a sample from
\(p_0(X_0\mid Y_0)\).

This is a screening-type approximation: once the high-SNR corrupted
observation \(Y_t\) is given, the original observation \(Y_0\) provides only
limited additional information about \(X_t\). This is encouraged by the
multi-speed design, where the condition channel is diffused more slowly than
the target channel. We therefore use
\begin{equation}\label{eq:screening}
p_t(X_t\mid Y_t,Y_0)\approx p_t(X_t\mid Y_t).
\end{equation}
which can be viewed as an approximate conditional-independence condition,
namely that \(X_t\) and \(Y_0\) are nearly independent given \(Y_t\),
\(X_t \perp Y_0\mid Y_t\). By Bayes' rule,
\begin{equation}
p_t(Y_0\mid X_t,Y_t)
=
\frac{p_t(X_t\mid Y_t,Y_0)p_t(Y_0\mid Y_t)}
     {p_t(X_t\mid Y_t)} .
\end{equation}
Therefore, \eqref{eq:screening} implies
\begin{equation}\label{eq:plugin-likelihood}
p_t(Y_0\mid X_t,Y_t)\approx p_t(Y_0\mid Y_t).
\end{equation}
(See Supplementary \textbf{Section 2.2} for more details).

Using Bayes' rule again, the conditional joint score decomposes as
\begin{equation}
\begin{aligned}\label{score}
  \nabla_{(x,y)}\log p_t(X_t,Y_t\mid Y_0)
  &=\nabla_{(x,y)}\log p_t(X_t,Y_t)\\
  &\quad+\nabla_{(x,y)}\log p_t(Y_0\mid X_t,Y_t)\\
  &\approx \nabla_{(x,y)}\log p_t(X_t,Y_t)\\
  &\quad+ \bigl(0,\nabla_{y}\log p_t(Y_0\mid Y_t)\bigr),
\end{aligned}
\end{equation}
where \((x,y)\) denotes a realization of \((X_t,Y_t)\) and
\(\nabla_{(x,y)}\) is the gradient with respect to \((x,y)\).

We approximate the unconditional joint score $\nabla_{(x,y)}\log p_t(X_t,Y_t)$
with a single neural network
\begin{equation}
  s_\theta(X_t,Y_t,t)\approx \nabla_{(x,y)}\log p_t(X_t,Y_t).
  \label{eq:joint-score-net}
\end{equation}
Conditioning is injected through the approximate plug-in correction term
\begin{equation}\label{eq:conditonal_term}
\begin{aligned}
c_y(t;Y_0,Y_t)\;:& \approx  \;\nabla_y\log p_t(Y_0\mid Y_t)\\
&=\frac{Y_0-Y_t}{\sigma_Y^2(t)},
\end{aligned}
\end{equation}
which admits an expression under Gaussian forward corruption
(Supplementary \textbf{Section 3} for more details). 

\paragraph{Pretrained-prior instantiation.}
The same plug-in principle can also be used when a joint score model is not
available. Suppose instead that we only have a pretrained unconditional
diffusion prior
\begin{equation}
s_\theta(X_t,t)\approx \nabla_x\log p_t(X_t).
\end{equation}
In this case, the observation is imposed directly on the noisy state \(X_t\).
For the inpainting case, the condition is simply a partial observation of the target:
\begin{equation}
Y_0 = M\odot X_0,
\end{equation}
where \(M\) is the binary observation mask. By Bayes' rule,
\begin{equation}
\nabla_x\log p_t(X_t\mid Y_0)
=
\nabla_x\log p_t(X_t)
+
\nabla_x\log p_t(Y_0\mid X_t),
\end{equation}
since \(p_t(Y_0)\) is independent of \(X_t\). The pretrained model provides the
first term, while the second term is approximated by a noisy-state
data-consistency correction. 
Specifically, we use the surrogate likelihood
\begin{equation}
p_t(Y_0\mid X_t)
\approx
\mathcal N(M\odot X_t,\sigma_X^2(t) I),
\end{equation}
which yields
\begin{equation}
\label{eq:pretrained-plugin}
\nabla_x\log p_t(Y_0\mid X_t)
\approx
M\odot
\frac{Y_0-M\odot X_t}{\sigma_X^2(t)} =
\frac{Y_0-Y_t}{\sigma_X^2(t)} .
\end{equation} 
Here \(M\odot M=M\).
This inpainting pretrained-prior case is closely related to the plug-in
conditioning mechanism above, and serves as a useful case. In the joint setting, the plug-in term acts on the condition
channel \(Y_t\), and \(X_t\) is influenced indirectly through the dependence of
the joint score component \(s_x(X_t,Y_t,t)\) on \(Y_t\). In the pretrained-prior
setting, there is no separately learned condition channel. Instead, for
inpainting, the condition is part of the target itself, since
\(Y_0=M\odot X_0\). Thus the plug-in correction can be applied directly to the
observed part of the noisy target state \(X_t\), equivalently to
\(Y_t=M\odot X_t\). Both cases follow the same principle: a learned diffusion
score is combined with an analytic observation-consistency correction at
sampling time.

\subsection{Plug-in conditional reverse dynamics (Contribution 2)}
\label{subsec:cond-dynamics}

We now derive the conditional reverse-time SDE and probability flow ODE induced by the unified conditioning in~\eqref{score}.
Let $Z_t=(X_t,Y_t)$ follow the joint forward SDE~\eqref{eq:joint-forward},

\subsubsection{Reverse-time SDE with plug-in conditioning}
The conditional reverse-time SDE that generates samples from the conditional marginals
$p^\star_t(\cdot,\cdot\mid Y_0)$, integrated backward in time from $t=T$ to $t=0$, is
\begin{equation}
dZ_t=
\Bigl(f(Z_t,t) - G(t)G(t)^\top \tilde s(Z_t,t;Y_0)\Bigr)\,dt
+ G(t)\,d\bar W_t .
\label{eq:cond-reverse-sde}
\end{equation}
where $\bar W_t$ is a reverse-time Brownian motion.

Writing \eqref{eq:cond-reverse-sde} in components yields
\begin{equation}
\left\{
\begin{aligned}
dX_t &= \Bigl(f_X(X_t,t) - g_X^2(t)\,\nabla_x\log p_t(X_t,Y_t)\Bigr)\,dt \\&+ g_X(t)\,d\bar W_t^{\,1},\\
dY_t &= \Bigl(f_Y(Y_t,t) - g_Y^2(t)\,\bigl[\nabla_y\log p_t(X_t,Y_t) \\&+ c_y(t;Y_0,Y_t)\bigr]\Bigr)\,dt + g_Y(t)\,d\bar W_t^{\,2}.
\end{aligned}
\right.
\label{eq:cond-reverse-sde-components}
\end{equation}
Here $p_t(X_t,Y_t)$ is the unconditional joint density used in the score terms, and in general $p_t^\star(\cdot,\cdot\mid Y_0)\neq p_t(\cdot,\cdot)$.


\subsubsection{Probability flow ODE with plug-in conditioning}
\label{Probability flow ODE with plug-in conditioning}
We view Fokker-Planck as a continuity argument, the probability flow ODE~\cite{anderson1982reverse} whose marginals exactly match $p^\star(\cdot,\cdot,t\mid Y_0)$ is
\begin{equation}
\begin{aligned}
\frac{dZ_t}{dt}
=f(Z_t,t) - 
\frac12\,G(t)G(t)^\top \nabla \log p^\star(Z_t,t\mid Y_0),
\label{eq:cond-pf-ode-pc-exact}
\end{aligned}
\end{equation}
However, $p^\star(\cdot,\cdot,t\mid Y_0)$ is intractable. 
We approximate its conditional score by a plug-in correction
that affects only the $y$-channel:
\begin{equation}
\begin{aligned}
\nabla_{(x,y)} \log p^\star(Z_t,t\mid Y_0)
\;&\approx\;
\nabla_{(x,y)} \log p_t(Z_t,t) \\\;+&\; (0,\;c_y(t;Y_0,Y_t)),
\end{aligned}
\label{eq:plugin-score-approx}
\end{equation}
where $c_y(t;Y_0,Y_t)\approx \nabla_y \log p_t(Y_0\mid Y_t)$.

Substituting \eqref{eq:plugin-score-approx} into \eqref{eq:cond-pf-ode-pc-exact} yields an implementable surrogate probability-flow ODE:
\begin{equation}
\begin{aligned}
\frac{dZ_t}{dt}
=&
\Bigl(f(Z_t,t)-G(t)G(t)^\top \tilde s(Z_t,t;Y_0)\Bigr)\\
+&\frac12\,G(t)G(t)^\top \bigl(\nabla\log p_t(Z_t,t)+c_y(t;Y_0,Y_t)\bigr).
\end{aligned}
\label{eq:cond-pf-ode-like-surrogate}
\end{equation}

i.e.,
\begin{equation}
\left\{
\begin{aligned}
\frac{dX_t}{dt} &= f_X(X_t,t) - \frac12\,g_X^2(t)\,\nabla_x\log p_t(X_t,Y_t),\\
\frac{dY_t}{dt} &= f_Y(Y_t,t) - \frac12\,g_Y^2(t)\,\bigl[\nabla_y\log p_t(X_t,Y_t)\\ &+ c_y(t;Y_0,Y_t)\bigr].
\end{aligned}
\right.
\label{eq:cond-pf-ode-components}
\end{equation}
\paragraph{Neural implementation.}
In practice, we approximate the unconditional joint score by a single network $s_\theta(X_t,Y_t,t)\approx \nabla_{(x,y)}\log p_t(X_t,Y_t)$ and inject conditioning by
\begin{equation}\label{eq:plugin-corrected-score}
\tilde s_\theta(X_t,Y_t,t;Y_0)
:=
s_\theta(X_t,Y_t,t)+(0,c_y(t;Y_0,Y_t)).
\end{equation}
Simulating \eqref{eq:cond-reverse-sde} or \eqref{eq:cond-pf-ode-components} from $t=T$ to $t=0$ and discarding $Y_t$ yields a valid conditional sample from $p(X_t\mid Y_0)$. Algorithm~\ref{alg:unified-conditional-sampling} summarizes the unified sampling procedure.

\subsection{Log-Fokker--Planck residual regularization (Contribution 3)}
\label{subsec:residual}

From subsection~\ref{Probability flow ODE with plug-in conditioning}, we obtain an approximate conditional probability-flow ODE and a matched reverse-time SDE driven by the same corrected score. In practice, the resulting ODE sampler may still lag behind its SDE counterpart in sample quality \cite{song2020score}. To reduce this discrepancy, 
we regularize the learned dynamics by encouraging the potential $u_\theta(x,y,t)$ to approximately satisfy the joint log-Fokker--Planck equation.
This log-FP residual regularization helps close the distributional gap between the plug-in ODE and SDE samplers, thereby improving the quality of the approximation probability-flow ODE sampling.
\paragraph{Log-FP operator.}
Let $u(x,y,t)=\log p(x,y,t)$ denote the log-density of the joint forward process~\eqref{eq:joint-forward}.
Under plug-in conditioning, the conditional reverse dynamics are driven by the corrected score
$\tilde s_\theta$, which corresponds to a corrected log-Fokker--Planck operator (see Supplementary \textbf{Section 4} for more details)
\begin{equation}
\mathcal{F}_{Y_0}[u]
:=\mathcal{F}[u]
- g_Y^2(t)\Bigl(\nabla_y\!\cdot c_y(t;Y_0,Y_t) + c_y(t;Y_0,Y_t)\cdot\nabla_y u\Bigr).
\label{eq:logfp-operator-corrected}
\end{equation}
where 
\begin{equation}
\begin{aligned}
\mathcal{F}[u]
&:=
\partial_t u
+\nabla_x\!\cdot f_X + f_X\!\cdot\nabla_x u
+\nabla_y\!\cdot f_Y + f_Y\!\cdot\nabla_y u\\
-&\frac12 g_X^2(t)\!\left(\Delta_x u+\|\nabla_x u\|^2\right)
-\frac12 g_Y^2(t)\!\left(\Delta_y u+\|\nabla_y u\|^2\right).
\label{eq:logfp-operator}
\end{aligned}
\end{equation}

\paragraph{Normalized residual.}
We measure the violation of the log-FP equation by the (time-averaged) $L^2$ residual
\begin{equation}
\mathcal{R}_{\mathrm{LFP}}(\theta,t)
:= V(T-t)^{-1}\int_t^T \|\mathcal{F}_{Y_0}[u_\theta](\cdot,\cdot,s)\|_{L^2(\Omega)}^2\,ds,
\label{eq:logfp-residual}
\end{equation}
where $V(r):=r\,\mathrm{Vol}(\Omega)$ and $\Omega=\Omega_X\times\Omega_Y$.
All derivatives in $\mathcal{F}[u_\theta]$ are computed by automatic differentiation; when needed, $\Delta u_\theta$ is estimated by the Hutchinson trace estimator.

\paragraph{Training objective.}
We augment the joint DSM loss with the log-FP residual regularizer:
\begin{equation}
\mathcal{L}_{\mathrm{total}}(\theta)
=
\mathcal{L}_{\mathrm{DSM}}^{\mathrm{joint}}(\theta)
+\alpha\,\mathbb{E}_{t\sim\mathrm{Unif}(0,T)}\bigl[\mathcal{R}_{\mathrm{LFP}}(\theta,t)\bigr].
\label{eq:loss-total-res}
\end{equation}

\paragraph{Closing the ODE--SDE gap.}
Under standard smoothness and growth assumptions (Supplementary \textbf{Section 4}),
a small residual implies that the approximation probability-flow ODE dynamics stay close to the reverse-time SDE dynamics in distribution:
\begin{theorem}[Residual controls the ODE--SDE gap]\label{thm:gap3}
For any $t\in[0,T]$, if $\mathcal{R}_{\mathrm{LFP}}(\theta,t)\le \delta$, then
\begin{equation}
W_2\!\left(p_\theta^{\mathrm{ODE}}(\cdot,\cdot,t),\,p_\theta^{\mathrm{SDE}}(\cdot,\cdot,t)\right)
\le C\,\delta,
\end{equation}
where $C>0$ is independent of $\delta$.
\end{theorem}
The proof follows an energy estimate for the joint log-Fokker--Planck equation.
Moreover, the plug-in terms cancel in the ODE--SDE comparison because the same correction
$-g_Y^2(\nabla_y\!\cdot c_y + c_y\cdot\nabla_y u)$ enters both the corrected log-FP equations; see Supplementary \textbf{Section 4}.

\begin{algorithm}[tb]
  \caption{Unified conditional sampling with a shared joint score (plug-in in the $y$-channel)}
  \label{alg:unified-conditional-sampling}
  \begin{algorithmic}
    \STATE {\bfseries Input:} observation $Y_0$, time grid $\{t_k\}_{k=0}^K$, joint score $s_\theta$, sampler $\in\{\mathrm{SDE},\mathrm{ODE}\}$
    \STATE {\bfseries Init:} sample $Z_T=(X_T,Y_T)\sim \pi_T$, 
where $\pi_T := \mathcal{N}(0,\sigma_x^2(T)I)\times \mathcal{N}(0,\sigma_y^2(T)I)$
\STATE {\bfseries (Inpainting)} Let $M$ be the binary mask (observed region $=1$). Set $Y_0 \leftarrow M\odot Y_0$ and initialize/project $y_T \leftarrow M\odot y_T$
(and inject $y$-noise only on $M$)
    \FOR{$k=K$ {\bfseries down to} $1$}
      \STATE $t \leftarrow t_k$, $\Delta t \leftarrow t_k - t_{k-1}$, $Z_t=(X_t,Y_t)$
      \STATE $(s_x,s_y) \leftarrow s_\theta(X_t,Y_t,t)$ \hfill // $s_x \approx \nabla_x\log p_t(X_t,Y_t)$ depends on $Y_t$
      \STATE $c_y \leftarrow \nabla_{y}\log p_t(Y_0\mid Y_t)$ \hfill // plug-in acts \emph{only} on $y$ (Eq.~\eqref{eq:conditonal_term})
      \STATE $\tilde s_x \leftarrow s_x$;\quad $\tilde s_y \leftarrow s_y + c_y$ \hfill // corrected score $\tilde s=(\tilde s_x,\tilde s_y)$ (Eq.~\eqref{eq:plugin-corrected-score})
      \IF{sampler = $\mathrm{SDE}$}
        \STATE $z_{t-\Delta t} \leftarrow \textsc{RevSDEStep}(z_t,t,\Delta t,(\tilde s_x,\tilde s_y))$
      \ELSE
        \STATE $z_{t-\Delta t} \leftarrow \textsc{FlowODEStep}(z_t,t,\Delta t,(\tilde s_x,\tilde s_y))$
      \ENDIF
      \STATE \hfill // Although $c_y$ does not enter $\tilde s_x$ explicitly, it influences $x$ indirectly through $Y_t$ in $s_x(X_t,Y_t,t)$.
    \ENDFOR
    \STATE {\bfseries Output:} $X_0$ (discard the $Y_t$-trajectory)
  \end{algorithmic}
\end{algorithm}

\section{Experiments}
\subsection{Conditional generation: method comparison}
\label{subsec:conditional-generation}

We evaluate the proposed plug-in conditioning mechanism on two conditional
generation tasks: inpainting and super-resolution.
The goal is to assess whether the proposed plug-in strategy
can consistently incorporate different forms of observed conditions into
score-based generative sampling.

\paragraph{Tasks and metrics.}
We consider two inverse problems in which the condition is a degraded version of
the target image. 
For \emph{inpainting}, the condition is produced by applying
randomly placed square masks covering \(25\%\) of the image area. For
\emph{super-resolution}, the condition is generated by bicubic downsampling, and
we perform \(8\times\) super-resolution from \(16\times16\) to \(128\times128\).

We report PSNR and SSIM~\cite{wang2004image} to measure distortion-level
fidelity. PSNR measures pixel-wise agreement with the ground truth, while SSIM captures local structural similarity beyond pixel-wise errors. To evaluate
perceptual quality, we also report LPIPS~\cite{zhang2018unreasonable}, which
computes distances in a deep feature space and is more aligned with human perceptual similarity than pixel-wise errors. Finally, we report joint Fr\'echet Inception
Distance (JFID)~\cite{batzolis2021conditional}, which applies the Fr\'echet distance to paired generated
samples and their corresponding conditions. This metric is included because
conditional generation should be evaluated not only by marginal image realism,
but also by consistency with the given observation.

All metrics are computed on the CelebA~\cite{liu2015deep} test set using the same number of
generated samples per condition. We follow the preprocessing protocol of
\citet{batzolis2021conditional}. For fairness, all score-based estimators use
the same backbone architecture, training data, preprocessing pipeline, sampling
budget, and evaluation protocol.

\paragraph{Baselines.}
We compare against four representative conditional generation strategies:
(i) CDE~\cite{song2020score}, a fixed-condition estimator that diffuses only
\(X_t\) while keeping \(Y\equiv Y_0\) fixed;
(ii) CDiffE~\cite{batzolis2021conditional}, a joint-diffusion estimator with
matched noise schedules for \((X_t,Y_t)\);
(iii) CMDE~\cite{batzolis2021conditional}, a multi-speed joint-diffusion
estimator where the condition channel is diffused with a smaller noise scale;
and (iv) HCFLOW~\cite{liang2021hierarchical}, a competitive task-specific
baseline for super-resolution.

\begin{figure*}[t]
\begin{center}
\includegraphics[width=0.7\linewidth]{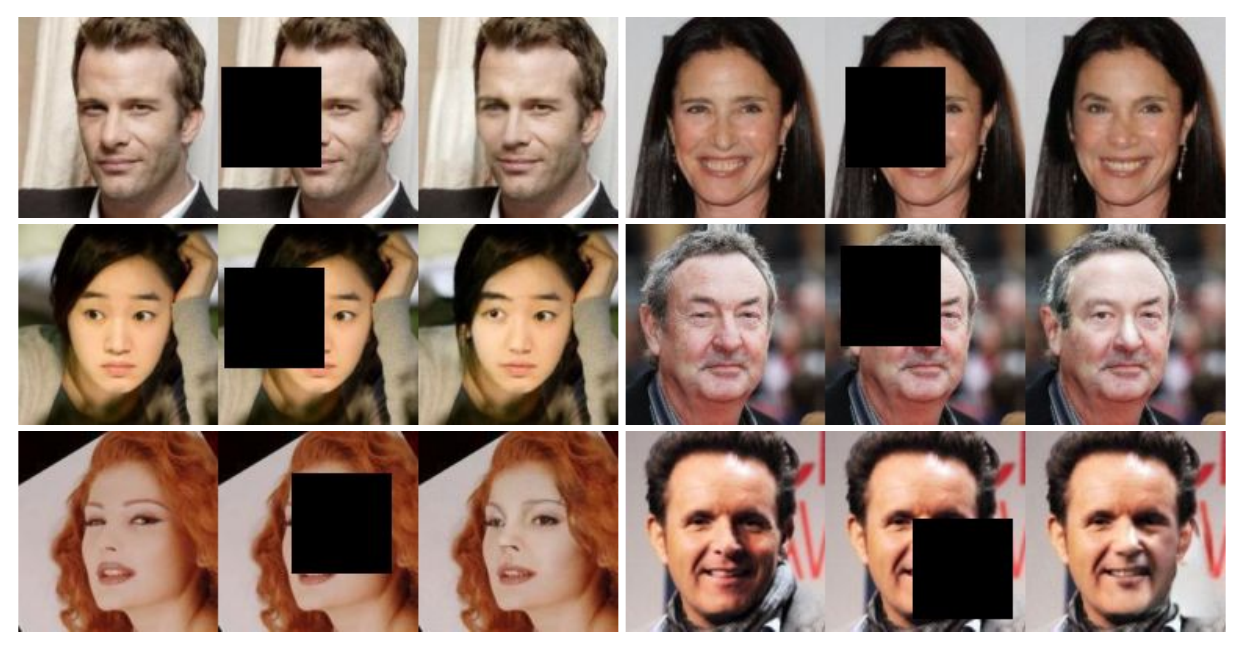}
\end{center}
\caption{\textbf{Inpainting on CelebA.}
For each example, we show the ground truth, the masked observation,
and samples generated by our method from left to right.}
\label{fig:impainting}
\end{figure*}

\paragraph{Noise schedules and plug-in correction.}
We use the VE-SDE parameterization for all score-based estimators. Unless
otherwise stated, the target channel uses \(\sigma_X(t)= 25\), while the
condition channel uses a smaller diffusion scale \(\sigma_Y(t)= 2.5\).
For the joint-diffusion baselines, we follow their original definitions:
CDiffE uses matched schedules with \(\sigma_Y(t)=\sigma_X(t)\), while CMDE uses
a reduced condition-channel noise scale~\cite{batzolis2021conditional}. Our
method uses the same multi-speed joint diffusion as CMDE, but differs at
sampling time by adding the analytic plug-in correction in the \(y\)-channel.
This isolates the effect of the proposed conditioning mechanism from the effect
of the noise schedule itself.

\begin{table}[t]
  \caption{Conditional generation results on CelebA inpainting. Lower LPIPS and
  JFID are better; higher PSNR and SSIM are better. All score-based methods use
  the same backbone and evaluation protocol.}
  \label{tab:conditional-results}
  \begin{center}
    \begin{small}
      \begin{sc}
        \begin{tabular}{lcccc}
          \toprule
          Method & PSNR $\uparrow$ & SSIM $\uparrow$ & LPIPS $\downarrow$ & JFID $\downarrow$ \\
          \midrule
          CDE    & 25.12 & 0.870 & \textbf{0.042} & 18.06 \\
          CDiffE & 23.07 & 0.844 & 0.057 & 19.25 \\
          CMDE   & 24.92 & 0.864 & 0.044 & 17.07 \\
          Ours   & \textbf{25.38} & \textbf{0.876} & 0.046 & \textbf{15.99} \\
          \bottomrule
        \end{tabular}
      \end{sc}
    \end{small}
  \end{center}
  \vskip -0.1in
\end{table}

\begin{table}[t]
  \caption{Conditional generation results on CelebA \(8\times\) super-resolution.
  Our method achieves the best PSNR, LPIPS, and JFID, while HCFLOW obtains the
  highest SSIM.}
  \label{tab:conditional-results-sr}
  \begin{center}
    \begin{small}
      \begin{sc}
        \begin{tabular}{lcccc}
          \toprule
          Method & PSNR $\uparrow$ & SSIM $\uparrow$ & LPIPS $\downarrow$ & JFID $\downarrow$ \\
          \midrule
          CDE    & 23.80 & 0.650 & 0.114 & 15.77 \\
          CDiffE & 23.83 & 0.656 & 0.139 & 20.20 \\
          CMDE   & 23.91 & 0.654 & 0.109 & 15.68 \\
          HCFLOW & 24.95 & \textbf{0.702} & 0.107 & 19.55 \\
          Ours   & \textbf{25.12} & 0.672 & \textbf{0.098} & \textbf{14.56} \\
          \bottomrule
        \end{tabular}
      \end{sc}
    \end{small}
  \end{center}
  \vskip -0.1in
\end{table}
\begin{figure*}[t]
\begin{center}
\includegraphics[width=0.7\linewidth]{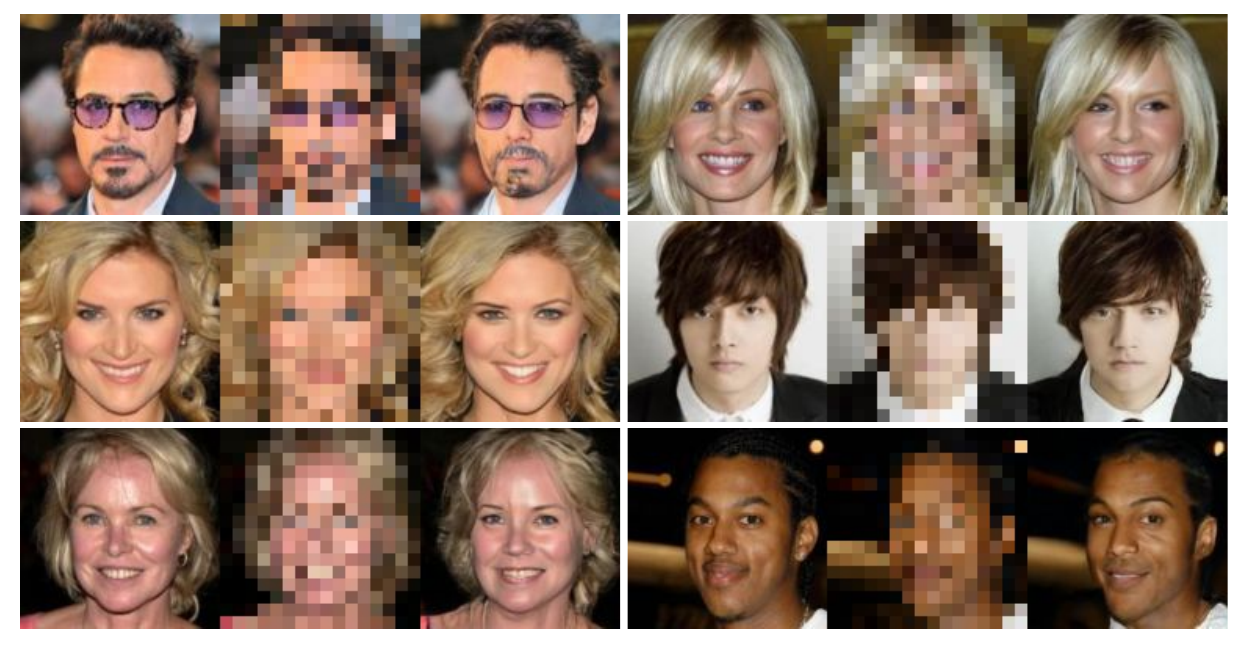}
\end{center}
\caption{\textbf{Super-resolution on CelebA.}
From left to right: ground-truth high-resolution images, low-resolution
observations, and samples generated by our method.}
\label{fig:super-resolution}
\end{figure*}

\paragraph{Results.}
Tables~\ref{tab:conditional-results} and~\ref{tab:conditional-results-sr}
summarize the quantitative results for inpainting and super-resolution,
respectively. Figures~\ref{fig:impainting} and~\ref{fig:super-resolution} show qualitative visual results for inpainting and super-resolution. On inpainting, our method achieves the best PSNR, SSIM, and JFID.
Although CDE obtains a slightly lower LPIPS, our method improves JFID from
\(18.06\) to \(15.99\), indicating better joint realism and consistency with the
observed condition. Compared with CMDE, which uses a similar multi-speed
diffusion setup but without the plug-in correction, our method improves both
distortion-level fidelity and JFID. This suggests that the gain is not only due
to using a smaller noise scale in the condition channel, but also due to the
analytic observation-consistency correction during sampling.

On super-resolution, our method achieves the best PSNR, LPIPS, and JFID.
HCFLOW obtains the highest SSIM, but its JFID is substantially worse than ours,
suggesting that stronger structural similarity does not necessarily imply
better conditional sample realism. Across both tasks, our method consistently
achieves the lowest JFID. This supports the main claim that plug-in conditioning
improves the alignment between generated samples and their observations while
maintaining competitive perceptual quality.

\subsection{Plug-in correction with a pretrained prior}
\label{subsec:pretrained-prior}

We next evaluate whether the proposed plug-in correction can be used with an
existing pretrained diffusion prior, without retraining the score network. The
pretrained diffusion model provides an unconditional image prior, while the
plug-in term is imposed only at sampling time.

\paragraph{Tasks and metrics.} We consider noisy random inpainting on \(256\times256\) images from FFHQ~\cite{karras2019style} and ImageNet256~\cite{russakovsky2015imagenet}. The mask is
sampled independently at the pixel level. For each image, the masking
probability is drawn uniformly from \([0.3,0.7]\), so that the missing ratio
varies across test samples. The observation is a masked image corrupted by
Gaussian noise with standard deviation \(\sigma_X(t)=0.05\). All methods are
evaluated using the same pretrained diffusion prior checkpoint, the same test
images, the same random masks and perturbations, and the same sampling budget
with \(\mathrm{NFE}=1000\). Metrics are computed on the same 100 validation
images. We report LPIPS, PSNR, and JFID as representative measures of
perceptual similarity, reconstruction fidelity, and conditional consistency,
respectively.

The main controlled comparison is with DPS~\cite{chung2022diffusion} and DiffPIR~\cite{zhu2023denoising}, since all methods use
the same pretrained prior checkpoint from DPS. In DPS, the likelihood term is
differentiated at every sampling step by backpropagating through the pretrained
U-Net, so it is marked as ``Grad.'' in Table~\ref{tab:pretrained-prior-inpainting}. DiffPIR and our
method do not require this U-Net backpropagation during sampling. DiffPIR uses a
closed-form proximal/denoising update for the data-consistency step, while our
method applies an analytic plug-in correction directly to the observed pixels.
Thus both methods incorporate the observation without computing gradients
through the pretrained score network. This avoids repeated gradient computation
through the pretrained U-Net and can therefore reduce the sampling overhead.

\begin{table}[t]
  \caption{Noisy random inpainting with a pretrained diffusion prior. All
  methods use the same pretrained prior, masks, perturbations, test images, and
  sampling budget \((\mathrm{NFE}=1000)\). Metrics are computed using \(100\)
  samples. We report LPIPS, PSNR, and JFID as representative measures of
  perceptual similarity, reconstruction fidelity, and conditional consistency,
  respectively.}
  \label{tab:pretrained-prior-inpainting}
  \begin{center}
    \begin{small}
      \begin{tabular}{llcccc}
        \toprule
        Dataset & Method & Grad. & LPIPS $\downarrow$ & PSNR $\uparrow$ & JFID $\downarrow$ \\
        \midrule
        FFHQ & DPS~\cite{chung2022diffusion}    & Yes & 0.2001 & 29.392 & 26.1210 \\
        FFHQ & DiffPIR~\cite{zhu2023denoising} & No & 0.2097 & 26.577 & \textbf{18.8032} \\
        FFHQ & Ours   & No & \textbf{0.1984} & \textbf{29.916} & 20.1224 \\
        \midrule
        ImageNet256 & DPS~\cite{chung2022diffusion}    & Yes & 0.3390 & 26.166 & 44.0984 \\
        ImageNet256 & DiffPIR~\cite{zhu2023denoising} & No & \textbf{0.2259} & 25.505 & 17.3510 \\
        ImageNet256 & Ours   & No & 0.2272 & \textbf{28.122} & \textbf{16.7309} \\
        \bottomrule
      \end{tabular}
    \end{small}
  \end{center}
  \vskip -0.1in
\end{table}

\paragraph{Results.}
Table~\ref{tab:pretrained-prior-inpainting} reports the results. Figure~\ref{fig:imagenet-inpainting} shows qualitative results on ImageNet256. Under the same
pretrained prior and sampling budget, our method consistently improves over DPS
in both reconstruction fidelity and conditional consistency. On FFHQ, it
improves PSNR from \(29.39\) to \(29.92\) and reduces JFID from \(26.12\) to
\(20.12\). On ImageNet256, the gain is more pronounced: JFID decreases from
\(44.10\) to \(16.73\), while PSNR increases from \(26.17\) to \(28.12\). These
improvements are obtained without the per-step U-Net backpropagation required by
DPS.

Compared with DiffPIR, our method is competitive rather than uniformly better
on every metric. DiffPIR achieves the best JFID on FFHQ, whereas our method
obtains the best LPIPS and PSNR. On ImageNet256, our method achieves the best
PSNR and JFID, while remaining close to DiffPIR in LPIPS. Our
framework offers a distinct perspective on conditional generation by viewing
conditioning as an analytic plug-in correction to a pretrained diffusion prior.
The competitive results suggest that this interpretation is both practically
effective and conceptually meaningful.
\begin{figure*}[t]
\begin{center}
\includegraphics[width=0.8\linewidth]{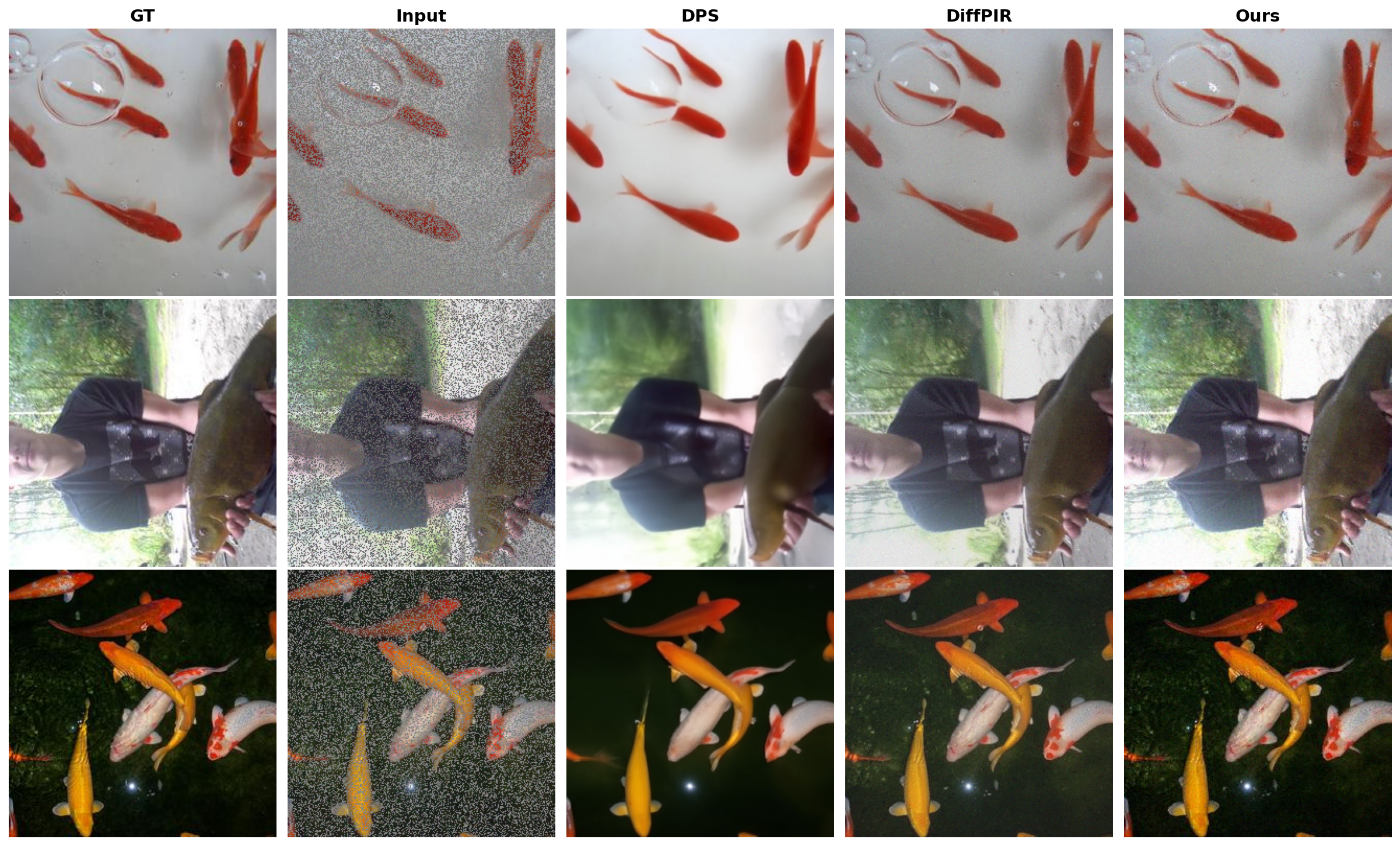}
\end{center}
\caption{\textbf{Qualitative comparison on ImageNet256 noisy random inpainting.}
From left to right: ground truth, noisy observation, DPS, DiffPIR, and Ours.
Compared with DPS and DiffPIR, our method produces reconstructions that are
more consistent with the observation while maintaining realistic image content.}
\label{fig:imagenet-inpainting}
\end{figure*}
\subsection{Closing the conditional ODE--SDE gap}
\label{sec:close-gap-exp}

We empirically evaluate whether the proposed conditional log--Fokker--Planck
residual reduces the discrepancy between the conditional probability-flow ODE
and the reverse-time SDE. Experiments are conducted on MNIST conditional
inpainting, where the left half of each image is observed and the right half is
generated.

\paragraph{Evaluation protocol.}
The purpose of this experiment is to isolate the ODE--SDE gap under a fixed
conditional generative model and sampling budget. We compare samples generated
by the conditional probability-flow ODE and the conditional reverse-time SDE.
The SDE sampler uses \(500\) Euler--Maruyama steps, while the ODE sampler uses a
fixed-step RK4 solver with a matched number of score evaluations. We measure the
sampler discrepancy using the sliced Wasserstein--\(2\) distance
\(W_2(\mathrm{ODE},\mathrm{SDE})\). To verify that reducing this discrepancy
does not degrade sample quality, we also report
\(W_2(\mathrm{ODE},\mathrm{GT})\) and \(W_2(\mathrm{SDE},\mathrm{GT})\), where
\(\mathrm{GT}\) denotes ground-truth test samples.

All evaluations are performed on \(N=2048\) test conditions. We compute the
sliced Wasserstein distance over \(8\) equal splits and report mean \(\pm\) std.
For reproducibility, the sampling seed is fixed for each epoch and split. For
each method, the reported checkpoint is selected by minimizing
\(W_2(\mathrm{ODE},\mathrm{SDE})\) under this evaluation protocol, and all
metrics are reported at the selected epoch.

\paragraph{Methods compared.}
We compare two training objectives using the same network architecture,
optimizer, data, and sampling procedures. The baseline is trained with the
standard conditional denoising score matching objective. CLOSE augments this
objective with the proposed conditional log--Fokker--Planck residual, using a
fixed residual weight \(\alpha=0.1\).

\begin{table}[t]
  \caption{MNIST conditional inpainting test results. Lower values are better.
  The checkpoint is selected by minimizing \(W_2(\mathrm{ODE},\mathrm{SDE})\).
  Results are reported as mean \(\pm\) std over \(8\) splits with \(N=2048\)
  test conditions.}
  \label{tab:mnist-w2}
  \begin{center}
    \begin{small}
      \begin{tabular}{lccc}
        \toprule
        Method
        & \(W_2(\mathrm{ODE},\mathrm{SDE}) \downarrow\)
        & \(W_2(\mathrm{ODE},\mathrm{GT}) \downarrow\)
        & \(W_2(\mathrm{SDE},\mathrm{GT}) \downarrow\) \\
        \midrule
        Baseline
        & \(0.0329 \pm 0.0009\)
        & \(0.0469 \pm 0.0016\)
        & \(0.0368 \pm 0.0013\) \\
        CLOSE (\(\alpha=0.1\))
        & \(\mathbf{0.0323 \pm 0.0011}\)
        & \(\mathbf{0.0452 \pm 0.0018}\)
        & \(\mathbf{0.0347 \pm 0.0011}\) \\
        \bottomrule
      \end{tabular}
    \end{small}
  \end{center}
  \vskip -0.1in
\end{table}

\paragraph{Results.}
Table~\ref{tab:mnist-w2} shows that adding the conditional
log--Fokker--Planck residual reduces the discrepancy between the ODE and SDE
samplers. CLOSE decreases \(W_2(\mathrm{ODE},\mathrm{SDE})\) from
\(0.0329\) to \(0.0323\), indicating a smaller empirical ODE--SDE gap under the
same sampling budget. Importantly, this reduction does not come at the cost of
sample quality. At the same selected checkpoint, CLOSE also improves
\(W_2(\mathrm{ODE},\mathrm{GT})\) and \(W_2(\mathrm{SDE},\mathrm{GT})\). This
supports the claim that the residual regularizer encourages better agreement
between the deterministic and stochastic conditional samplers while preserving,
and slightly improving, their alignment with the ground-truth conditional
distribution.

\paragraph{Scope of the comparison.}
We emphasize that this experiment is intended as a controlled validation
of the residual regularization rather than evidence of a uniform reduction
of the ODE--SDE gap across high-dimensional conditional generation tasks.
MNIST enables distribution-level evaluation using sliced \(W_2\) over a
large number of conditional samples while keeping residual-based training
computationally tractable. Across eight data splits, CLOSE consistently
improves both \(W_2(\mathrm{ODE},\mathrm{GT})\) and
\(W_2(\mathrm{SDE},\mathrm{GT})\); however, the reduction in
\(W_2(\mathrm{ODE},\mathrm{SDE})\) is checkpoint-dependent.
We therefore interpret these results as evidence for the regularization
effect in this controlled setting.
\section{Conclusion and Discussion}

We proposed a plug-in conditioning framework for conditional score-based
diffusion models. The method learns a joint score field and imposes conditioning
at inference through an analytic plug-in correction term, providing a
transparent view of conditioning as learned diffusion dynamics plus an
observation-consistency correction. We also showed that the same principle
extends naturally to pretrained-prior settings.

Building on this formulation, we derived conditional reverse-time SDEs and
matched approximate probability-flow ODEs. We further introduced a
log-Fokker--Planck residual to analyze and reduce the conditional ODE--SDE
discrepancy. Experiments on inpainting and super-resolution demonstrate
competitive performance, while controlled ODE--SDE experiments provide
evidence for the effectiveness of residual regularization in improving
deterministic ODE sampling.

Several limitations remain. The closed-form plug-in correction relies on a
linear-Gaussian forward kernel for the condition channel. The same construction
can be extended to VP and sub-VP processes, whose transition kernels are also
Gaussian, by using their corresponding conditional scores. For non-Gaussian
forward corruptions or nonlinear observation processes, a closed-form correction
may not be available and an approximate likelihood score may instead be required.
Moreover, the screening approximation may become less accurate when the
conditional posterior is strongly multimodal or when the condition-noise level
is high. This motivates the slower condition diffusion used in our multi-speed
formulation. The effectiveness of the plug-in correction also depends on the
diffusion schedules and the accuracy of the learned joint score.

Our formulation is conceptually complementary to classifier-free guidance
(CFG)~\cite{ho2022classifier}: CFG combines conditional and unconditional
score estimates to strengthen conditioning, whereas our framework derives
an analytic observation-consistency correction from the forward corruption
model. These mechanisms act on different aspects of conditional sampling
and may in principle be combined, which we leave for future investigation.
Finally, we use uniform sampling of $t$, which provides an unbiased Monte
Carlo estimator of the uniformly time-averaged log-Fokker--Planck residual
used in our objective. Adaptive strategies such as SNR- or residual-weighted
time sampling may improve optimization efficiency. Exploring these extensions
and developing tighter quantitative characterizations of the conditional
ODE--SDE discrepancy are promising directions for future work.

\bibliographystyle{plainnat}
\bibliography{example_paper}
\clearpage
\appendix

\section*{Supplementary Material}
\setcounter{section}{0}
\renewcommand{\thesection}{\arabic{section}}

\section{Conditional Independence Under Independently Corrupted VE Forwards}
\label{app:cond-indep}
This appendix justifies the conditional independence claim used in
main paper.
Under a VE-SDE forward backbone, we corrupt $X_0$ and the observation $Y_0$
with \emph{independent} Gaussian noises:
\begin{equation}
\label{eq:ve-linear}
\begin{aligned}
X_t &= X_0 + \sigma_X(t)\,Z_1,\\
Y_t &= Y_0 + \sigma_Y(t)\, Z_2,
\end{aligned}
\end{equation}
where $Z_1,Z_2\sim\mathcal{N}(0,I)$ are independent and
$(Z_1,Z_2)\perp (X_0,Y_0)$.

\begin{proposition}[Conditional independence]
\label{prop:cond-indep}
Under \eqref{eq:ve-linear}, we have $X_t \perp Y_t \mid Y_0$ and hence
\begin{equation}
\label{eq:factor}
p_t(X_t,Y_t\mid Y_0)=p_t(X_t\mid Y_0)\,p_t(Y_t\mid Y_0).
\end{equation}
\end{proposition}

\begin{proof}
Fix $Y_0=y_0$. Then
$Y_t=y_0+\sigma_Y(t)Z_2$ depends only on $Z_2$.
Meanwhile, $X_t=X_0+\sigma_X(t)Z_1$ depends only on $(X_0,Z_1)$,
which is independent of $Z_2$ by assumption.
Therefore $X_t$ and $Y_t$ are independent given $Y_0$, i.e.\
$X_t \perp Y_t \mid Y_0$.
The density factorization \eqref{eq:factor} is the definition of
conditional independence.
\end{proof}

\section{Score Decomposition and the Plug-in Approximation in VE Conditional Sampling}
\label{app:score-decomp}

This appendix derives the score decomposition and provides a quantitative justification for the plug-in approximation in the
VE-SDE noise regime.

\subsection{Exact identity via Bayes' rule}

Assume that the conditional density $p_t(X_t,Y_t\mid Y_0)$ and the unconditional density  $p_t(X_t,Y_t)$ exist and are positive.
Then Bayes' rule gives, for $(X_t,Y_t)$,
\begin{equation}
\label{eq:bayes-basic}
\begin{aligned}
\log p_t(X_t,Y_t\mid Y_0)
=
\log p_t(X_t,Y_t)
+\log p_t(Y_0\mid X_t,Y_t)-\log p(Y_0).
\end{aligned}
\end{equation}
Taking the gradient with respect to $(X_t,Y_t)$ and using that $\log p(Y_0)$ is
constant in $(X_t,Y_t)$ yields the exact score decomposition
\begin{equation}
\label{eq:score-exact}
\begin{aligned}
\nabla_{(x,y)}\log p_t(X_t,Y_t\mid Y_0)
=
\nabla_{(x,y)}\log p_t(X_t,Y_t)+\nabla_{(x,y)}\log p_t(Y_0\mid X_t,Y_t),
\end{aligned}
\end{equation}

\subsection{Plug-in term from conditional independence: a derivation}
\label{app:plugin-derivation}

We justify the plug-in replacement
$p_t(Y_0\mid X_t,Y_t) \approx p_t(Y_0\mid Y_t)$
(and hence
$\nabla_{(x,y)}\log p_t(Y_0\mid X_t,Y_t)\approx (0,\nabla_y\log p_t(Y_0\mid Y_t))$)
via the implication
\[
\text{conditional independence } \Longrightarrow \text{Bayes identity } \Longrightarrow \text{plug-in term}.
\]

\paragraph{Step 1: Assumption.}
Fix a time $t\in(0,T]$ such that the plug-in correction is applied.
We assume that, conditioned on the current corrupted observation $Y_t$,
the original observation $Y_0$ provides only negligible additional information about $X_t$.
Formally, we assume the approximation
\begin{equation}
\label{eq:ci-assump}
p_t(X_t \mid Y_t, Y_0)\ \approx\ p_t(X_t \mid Y_t).
\end{equation}
We interpret \eqref{eq:ci-assump} quantitatively by assuming that the conditional likelihood ratio is close to one:
there exists a small error $\varepsilon_t\ge 0$ such that
\begin{equation}
\label{eq:screening-lr}
\left|\log \frac{p_t(X_t \mid Y_t, Y_0)}{p_t(X_t \mid Y_t)}\right|
\le \varepsilon_t.
\end{equation}
Moreover, we model $\varepsilon_t$ as controlled by the $y$-channel noise level, i.e.,
$\varepsilon_t = \omega(\sigma_Y(t))$ for some nondecreasing function $\omega$ with $\omega(0)=0$.
In particular, the conditional dependence between $X_t$ and $Y_0$ given $Y_t$ is weak in this regime,
which can be quantified, for instance, by a small conditional mutual information $I_t(X_t;Y_0\mid Y_t)$.

This can be viewed as an approximate conditional independence assumption,
namely,
\[
X_t \perp Y_0 \mid Y_t .
\]
By the Bayes identity,
\[
p_t(Y_0\mid X_t,Y_t)
=
\frac{p_t(X_t\mid Y_t,Y_0)p_t(Y_0\mid Y_t)}
     {p_t(X_t\mid Y_t)} .
\]
Therefore, under the approximation
\(p_t(X_t\mid Y_t,Y_0)\approx p_t(X_t\mid Y_t)\), we obtain
\[
p_t(Y_0\mid X_t,Y_t)\approx p_t(Y_0\mid Y_t).
\]

\noindent\textbf{Heuristic justification.}
Under VE corruption, $Y_t = Y_0 + \sigma_Y(t)\,Z_2$ with $Z_2\sim\mathcal{N}(0,I)$.
When $\sigma_Y(t)$ is small (high-SNR), $Y_t$ already localizes $Y_0$ sharply,
so conditioning on $Y_0$ in addition to $Y_t$ yields only a minor refinement.
In our setting, this regime is enforced by choosing a smaller diffusion intensity in the $y$-channel
(e.g.\ $g_Y(t)\ll g_X(t)$, or $\sigma_Y(t)=\beta\sigma_X(t)$ with $\beta\ll 1$),
so that $Y_t$ remains close to $Y_0$ throughout most of the reverse trajectory.

\subsection{Exact validation of the screening approximation}
\label{app:screening-validation}

The screening relation in \eqref{eq:ci-assump} is a working approximation
adopted by our framework rather than an exact identity. We therefore
quantify the plug-in approximation error in a tractable setting where the
true conditional score is available in closed form.

For jointly Gaussian $(X_0,Y_0)$ under the VE corruption schedule, the
approximation error can be decomposed into two components: (i) the screening
error in the $x$-channel induced by
$p_t(X_t\mid Y_t,Y_0)\approx p_t(X_t\mid Y_t)$, and (ii) the error associated
with replacing the full conditional correction
$\nabla_y\log p_t(Y_0\mid X_t,Y_t)$ by the $y$-channel plug-in term
$\nabla_y\log p_t(Y_0\mid Y_t)$. The latter contribution is negligible
relative to $c_y$ in the high-SNR regime.

For the inpainting-like observation $Y_0=M\odot X_0$,
Table~\ref{tab:screening} reports the worst-case-over-$t$ score error of
the screening approximation together with the end-to-end relative $W_2$
distance between the resulting samples and the analytic posterior. For
reference, the exact-score sampler gives a relative $W_2$ floor of $0.003$.
\begin{table}[t]
\caption{Exact plug-in approximation error as a function of the condition
diffusion speed ratio $\beta$. The empirical score error decays approximately
as $O(\beta^{0.98})$.}
\label{tab:screening}
\vspace{2mm}
\centering
\footnotesize
\setlength{\tabcolsep}{4pt}
\begin{tabular}{lccccc}
\toprule
$\beta=\sigma_Y/\sigma_X$ & 0.02 & 0.05 & \textbf{0.1} & 0.3 & 1.0 \\
\midrule
Score error & 0.020 & 0.050 & \textbf{0.099} & 0.287 & 0.706 \\
End-to-end rel.\ $W_2$ & 0.008 & 0.020 & \textbf{0.045} & 0.138 & 0.414 \\
\bottomrule
\end{tabular}
\end{table}

The approximation error decreases rapidly as $\beta$ decreases. At our
default $\beta=0.1$, the worst-case score error is $0.099$, while the final
distribution remains within $4.5\%$ relative $W_2$ of the analytic posterior.
This indicates that the local approximation errors do not substantially
accumulate through the complete reverse sampler.

\subsection{Ablation on the condition diffusion speed $\beta$}
\label{app:ablation-beta}

We study the effect of the condition-channel diffusion speed ratio $\beta$ (relative to the state channel) in our multi-speed joint diffusion.
Table~\ref{tab:ablation-beta} reports quantitative results on conditional inpainting.

\begin{table}[h]
\centering
\caption{Ablation on $\beta$ (condition diffusion speed ratio). Higher is better for PSNR/SSIM, lower is better for LPIPS.}
\label{tab:ablation-beta}
\begin{tabular}{lccc}
\toprule
$\beta$ & PSNR $\uparrow$ & SSIM $\uparrow$ & LPIPS $\downarrow$ \\
\midrule
0.05 & 24.02 & 0.912 & 0.061 \\
0.10 & \textbf{25.38} & 0.876 & \textbf{0.046} \\
0.20 & 25.04 & 0.854 & 0.058 \\
0.30 & 24.36 & 0.802 & 0.067 \\
\bottomrule
\end{tabular}
\end{table}

Overall, $\beta=0.1$ provides the best empirical trade-off, achieving the
highest PSNR and the lowest LPIPS. Interestingly, although the exact analysis
above shows that the plug-in approximation is more accurate at $\beta=0.05$,
this does not translate into better end-to-end performance: $\beta=0.1$
performs better in both PSNR and LPIPS. Approximation fidelity alone therefore
does not determine reconstruction quality, suggesting a trade-off between
approximation accuracy and optimization or sampling behavior. We use
$\beta=0.1$ as the default setting in all main experiments.

\paragraph{Step 2: Bayes rule under the assumption.}
By using the definition of the conditional probability, we have the following identities
$$
\begin{aligned}
p_t(Y_0\mid X_t,Y_t)
&=\frac{p_t(Y_0,X_t,Y_t)}{p_t(X_t,Y_t)}=\frac{p_t(X_t\mid Y_0,Y_t)\,p_t(Y_0,Y_t)}{p_t(X_t\mid Y_t)\, p_t(Y_t)}\\
&=\frac{p_t(X_t\mid Y_0,Y_t)\,p_t(Y_0\mid Y_t)}{p_t(X_t\mid Y_t)}.
\end{aligned}
$$
Recall our assumption \eqref{eq:ci-assump}, then we conclude that
$$
p_t(Y_0\mid X_t,Y_t)\approx p_t(Y_0\mid Y_t).
$$

\section{Closed-form Plug-in Term}
\label{app:plugin}
For the VE-SDE forward process in the $y$-channel (with zero drift),
the marginal perturbation admits the linear-Gaussian form
\begin{equation}
\label{eq:ve-marginal-y}
Y_t = Y_0 + \sigma_Y(t)\,Z_2,
\qquad Z_2\sim\mathcal{N}(0,I),
\end{equation}
where
\[
\sigma_Y^2(t)=\int_0^t g_Y^2(s)\,ds.
\]
Rearranging \eqref{eq:ve-marginal-y} gives
\[
Y_0 = Y_t - \sigma_Y(t)\,Z_2.
\]
Hence the conditional distribution is Gaussian:
\[
Y_0 \mid Y_t \sim \mathcal{N}\!\bigl(Y_t,\ \sigma_Y^2(t)\,I\bigr).
\]
Its log-density (up to an additive constant independent of $Y_t$) is
\[
\log p_t(Y_0\mid Y_t)
= -\frac{1}{2\sigma_Y^2(t)}\|Y_0-Y_t\|_2^2
-\frac{d}{2}\log\!\bigl(2\pi\sigma_Y^2(t)\bigr).
\]
Taking the gradient with respect to $Y_t$ yields
\[
\nabla_{Y_t}\log p_t(Y_0\mid Y_t)
= \frac{Y_0-Y_t}{\sigma_Y^2(t)}.
\]
Therefore, the plug-in correction used in the main text is
\[
c_y(t;Y_0,Y_t)
:=\nabla_{y}\log p_t(Y_0\mid Y_t)
=\frac{Y_0-Y_t}{\sigma_Y^2(t)}.
\]

\section{Proof of Theorem}
\label{app:residual-proof}
To obtain samples from \( p(X_t, Y_t \mid Y_0) \), the conditional term \( p(Y_0 \mid Y_t)\) 
affects only the $y$-component of the conditional score.
Hence we first derive the reverse SDE / probability-flow ODE for the joint score
$\nabla_{(x,y)}\log p_t(x,y)$, and then incorporate the plug-in correction in the $y$-drift.

The score function represents the gradient of the log-probability density with respect to the state variables $X_t$ and $Y_t$. Specifically, this score function can be split into two components
\begin{equation}\label{eq:joint-score-split}
\nabla_{(x,y)} \log p_t(x,y)=\bigl(\nabla_x \log p_t(x,y),\,\nabla_y \log p_t(x,y)\bigr).
\end{equation}
where 
 $\nabla_{x} \log p\left(X_t, Y_t\right)$  is the gradient of the log-probability with respect to $x$ and
$\nabla_{y} \log p\left(X_t, Y_t\right)$ is the gradient of the log-probability with respect to $y$. This allow us to formulate the following SDEs
\begin{equation}\label{SDEs}
\left\{
\begin{aligned}
dX_t &= \Bigl(f_X(X_t,t)-g_X^2(t)\nabla_x\log p_t(X_t,Y_t)\Bigr)\,dt + g_X(t)\,d\bar W_t^{\,1},\\
dY_t &= \Bigl(f_Y(Y_t,t)-g_Y^2(t)\nabla_y\log p_t(X_t,Y_t)\Bigr)\,dt + g_Y(t)\,d\bar W_t^{\,2}.
\end{aligned}
\right.
\end{equation}
In the unified conditional sampler, the $y$-equation is corrected by the plug-in term:
\begin{equation}\label{eq:rev-sde-joint-plugin}
dY_t
=
\Bigl(
f_Y(Y_t,t)
- g_Y^2(t)\nabla_y\log p_t(X_t,Y_t)
- g_Y^2(t)c_y(t;Y_0,Y_t)
\Bigr)\,dt
+ g_Y(t)\,d\bar W_t^{\,2}.
\end{equation}

Here, $X_t, Y_t$ are adapted to the reverse time filtration $\overline{\mathcal{F}}_t$ and according Assumption 1, has a unique solution that is continuous in time. The marginal density of $X_t,Y_t$  is denoted by $p(\cdot,\cdot, t)$ at time $t$. We assume it is equipped with some terminal distribution $\pi$ of $X_t,Y_t$ at time $T$, i.e., $p(\cdot,\cdot, T)=\Pi$, where $\Pi$ is chosen as a Gaussian approximation of $p(\cdot,\cdot, T)$.

The corresponding ODEs for the score function, derived from the Fokker-Planck equation for the reverse process, are given by
\begin{equation}\label{Odes}
\begin{cases}
& \displaystyle\frac{d X_t}{d t} = f_X(X_t, t) - \frac{1}{2} g_X^2(t) \nabla_x \log p_t(X_t, Y_t), \vspace{0.2cm}\\ 
& \displaystyle\frac{d Y_t}{d t} = f_Y(Y_t, t) - \frac{1}{2} g_Y^2(t) \nabla_y \log p_t(X_t, Y_t),
\end{cases}
\end{equation}
In the unified conditional sampler, the plug-in correction modifies only the $y$-drift:
\begin{equation}\label{eq:pf-ode-joint-plugin-y}
\frac{dY_t}{dt}
=
f_Y(Y_t,t)
-\frac12\,g_Y^2(t)\Bigl(\nabla_y\log p_t(X_t,Y_t)+c_y(t;Y_0,Y_t)\Bigr).
\end{equation}
and it is equipped with  terminal distribution $p(\cdot,\cdot, T)$ for $X_t,Y_t$.

Now we have derived the reverse joint SDEs \eqref{SDEs} and ODEs \eqref{Odes}, we proceed by following the log-Fokker--Planck residual strategy of Deveney et al. We leverage the log-Fokker-Planck residual to help close the gap in the unified method.

The log-Fokker-Planck residual serves as a correction term that aligns the dynamics of the ODE-based process with the more accurate SDE-based process. By incorporating this residual into the training loss, we can ensure that the unified method's reverse process accurately reflects the target conditional distribution. This allows us to reduce the approximation error between the two methods and improve the overall performance of the unified model.

To formalize the alignment between the ODEs and SDEs dynamics in the unified method, we begin by defining the approximate log-probability densities for the joint system involving $X_t$ and $Y_t$.
Let $u_\theta(x,y,t):=\log p_\theta(x,y,t)$ be a neural approximation of the log-density on
$\Omega_X\times\Omega_Y\times[0,T]$.
We write $\nabla_x u_\theta$ and $\nabla_y u_\theta$ for its partial gradients.
Using this, we can derive approximate versions
of the reverse SDE \eqref{SDEs} and its deterministic flow \eqref{Odes}.
To simplify our notation, we define the approximate reverse drift as
\begin{equation}\label{joint eq:sde}
\left\{
\begin{aligned}
f_{X,\theta}^{\mathrm{SDE}}(x,y,t) &= f_X(x,t)-g_X^2(t)\nabla_x u_\theta(x,y,t),\\
f_{Y,\theta}^{\mathrm{SDE}}(x,y,t) &= f_Y(y,t)-g_Y^2(t)\nabla_y u_\theta(x,y,t).
\end{aligned}
\right.
\end{equation}

which is obtained by substituting the potential model into the drift term of \eqref{SDEs}. According to Assumption 1, we have $f_{X,\theta}^{S D E}, f_{Y,\theta}^{S D E} \in C^{\infty}$ and 
$$\left\|f_{X,\theta}^{\mathrm{SDE}}(x,y, t)\right\|_2 \leq\left(K_f+M^2 K_u\right)\left(1+\|(x,y)\|_2\right),
\quad
\left\|f_{Y,\theta}^{\mathrm{SDE}}(x,y, t)\right\|_2 \leq\left(K_f+M^2 K_u\right)\left(1+\|(x,y)\|_2\right).$$

Plug \eqref {joint eq:sde} into \eqref{SDEs} to get 

\begin{equation}\label{SDE2}
\left\{
\begin{aligned}
dX_t &= f_{X,\theta}^{\mathrm{SDE}}(X_t,Y_t,t)\,dt + g_X(t)\,d\bar W_t^{\,1},\\
dY_t &= f_{Y,\theta}^{\mathrm{SDE}}(X_t,Y_t,t)\,dt + g_Y(t)\,d\bar W_t^{\,2}.
\end{aligned}
\right.
\end{equation}

Applying the reverse-time SDE formula (e.g., \cite{anderson1982reverse}),
the density $p_\theta^{\mathrm{SDE}}$ associated with \eqref{SDE2} satisfies a forward-time Fokker--Planck equation.
Equivalently, one may write a forward-time SDE whose drift involves $\nabla \log p_\theta^{\mathrm{SDE}}$, namely \eqref{reverse:SDE2}.

\begin{equation}\label{reverse:SDE2}
\left\{
\begin{aligned}
dX_t &= \Bigl(f_{X,\theta}^{\mathrm{SDE}}(X_t,Y_t,t)
            + g_X^2(t)\,\nabla_x \log p_{\theta}^{\mathrm{SDE}}(X_t,Y_t,t)\Bigr)\,dt
        + g_X(t)\,dW_t^{1},\\
dY_t &= \Bigl(f_{Y,\theta}^{\mathrm{SDE}}(X_t,Y_t,t)
            + g_Y^2(t)\,\nabla_y \log p_{\theta}^{\mathrm{SDE}}(X_t,Y_t,t)\Bigr)\,dt
        + g_Y(t)\,dW_t^{2}.
\end{aligned}
\right.
\end{equation}

where initial state $X_0,Y_0$ is drawn from $p_\theta^{SDE}(\cdot,\cdot, 0)$. 
 For the forward SDE \eqref{reverse:SDE2} the density $P_\theta^{SDE}$ follows the  the forward Fokker–Planck
equation
\begin{equation}\label{new_fp sde2|y}
\begin{aligned}
\frac{\partial p_\theta^{\mathrm{SDE}}}{\partial t}(x,y,t)
&= -\nabla_x \cdot \Bigl(f_{X,\theta}^{\mathrm{SDE}}(x,y,t)\,p_\theta^{\mathrm{SDE}}(x,y,t)\Bigr)
   -\nabla_y \cdot \Bigl(f_{Y,\theta}^{\mathrm{SDE}}(x,y,t)\,p_\theta^{\mathrm{SDE}}(x,y,t)\Bigr)\\
&\quad + \frac{1}{2} g_X^2(t)\,\Delta_x p_\theta^{\mathrm{SDE}}(x,y,t)
      + \frac{1}{2} g_Y^2(t)\,\Delta_y p_\theta^{\mathrm{SDE}}(x,y,t).
\end{aligned}
\end{equation}

 equipped with the terminal condition $\Pi$ on the full space, i.e., $p_\theta^{S D E}(x,y, T)=\Pi(x,y)$. Considering \eqref{new_fp sde2|y}, we introduce positive Dirichlet boundary conditions. Let $p_B^{S D E}: \partial \Omega_X \times \partial \Omega_Y \times[0, T] \rightarrow \mathbb{R}$ denote a positive function which is equal to $p_\theta^{S D E}$ on $\partial \Omega_X \times \partial \Omega_Y $. We obtain the approximate Fokker-Planck equation \eqref{new_fp sde2|y} on the domain $\partial \Omega_X \times \partial \Omega_Y$ with initial data $p_0(\cdot,\cdot,0)$ restricted to $\partial \Omega_X \times \partial \Omega_Y$ and Dirichlet boundary conditions $p_B^{S D E}$ on $\partial \Omega_X \times \partial \Omega_Y \times[0, T]$.

To complete the analysis of the Fokker–Planck equations for the probability densities, one can also consider the log-Fokker–Planck equations, which describe the evolution of the potential function. The log-density \( u_\theta^{S D E}(x,y,t) = \log p_\theta^{S D E}(x,y,t) \), where \( p_\theta^{S D E} \) is the solution to \eqref{new_fp sde2|y} in forward time, satisfies the approximate log-Fokker–Planck equation (in forward time) given by 
\begin{equation}\label{u_sde 2}
\begin{aligned}
\frac{\partial u_\theta^{SDE}}{\partial t} = &-\nabla_x \cdot f_{X,\theta}^{SDE} - f_{X,\theta}^{SDE} \cdot \nabla_x u_\theta^{SDE} - \nabla_y \cdot f_{Y,\theta}^{SDE} - f_{Y,\theta}^{SDE}\cdot \nabla_y u_\theta^{SDE}\\
 &-\frac{1}{2} g_X^2\left(\Delta_x u_\theta^{SDE}+\left\|\nabla_x u_\theta^{SDE}\right\|^2\right)-\frac{1}{2} g_Y^2\left(\Delta_y u_\theta^{SDE}+\left\|\nabla_y u_\theta^{SDE}\right\|^2\right).
\end{aligned}
\end{equation}
We equip \eqref{u_sde 2} with terminal data $u_\theta^{SDE}(\cdot, \cdot,T)$ restricted to $ \Omega_X \times \Omega_Y $ and boundary conditions $u_B^{S D E}=\log p_B^{S D E}$ on $\partial \Omega_X \times \partial \Omega_Y \times[0, T]$.
 We obtain the approximate log-Fokker-Planck equation \eqref{u_sde 2} on the domain $\partial \Omega_X \times \partial \Omega_Y$ with initial data $u_0(\cdot,\cdot,0)$ restricted to $\partial \Omega_X \times \partial \Omega_Y$ and Dirichlet boundary conditions $u_B^{S D E}$ on $\partial \Omega_X \times \partial \Omega_Y \times[0, T]$.

For the forward SDE \eqref{eq:joint-forward}, the density obeys the forward Fokker–Planck
equation
\begin{equation}\label{new_fp1 sde|y}
\begin{aligned}
\frac{\partial p}{\partial t}(x, y,t )
&= -\nabla_x \cdot \bigl(f_X(x, t)\,p(x, y,t)\bigr)
   - \nabla_y \cdot \bigl(f_Y(y, t)\,p(x, y,t)\bigr) \\
&\quad + \frac{1}{2} g_X^2(t)\,\Delta_x p(x, y,t)
      + \frac{1}{2} g_Y^2(t)\,\Delta_y p(x, y,t).
\end{aligned}
\end{equation}
 equipped with the initial data $p(\cdot,\cdot,0)$ on the whole space. Considering \eqref{new_fp1 sde|y}, we introduce positive Dirichlet boundary conditions. Let $p_B: \partial \Omega_X \times \partial \Omega_Y \times[0, T] \rightarrow \mathbb{R}$ denote a positive function which is equal to $p$ on $\partial \Omega_X \times \partial \Omega_Y $. Notice that \eqref{new_fp sde2|y} and \eqref{new_fp1 sde|y} share the same transport structure.
The difference lies in which drift is used (the learned reverse drift versus the true forward drift),
while the diffusion operator keeps the standard forward sign.

From the density  satisfying the forward
Fokker–Planck equation \eqref{new_fp1 sde|y} of the forward SDE \eqref{eq:joint-forward} and the associated potential $u = \log p$,
we get the forward log-Fokker–Planck equation  as
\begin{equation}\label{u1_sde 2}
\begin{aligned}
\frac{\partial u}{\partial t}(x,y,t) = &-\nabla_x \cdot f_X(x,t) - f_X(x,t)\cdot \nabla_x u - \nabla_y \cdot f_Y(y,t) - f_Y(y,t)\cdot \nabla_y u\\
- &\frac{1}{2} g_X^2\left(\Delta_x u+\left\|\nabla_x u\right\|^2\right)-\frac{1}{2} g_Y^2\left(\Delta_y u+\left\|\nabla_y u\right\|^2\right).
\end{aligned}
\end{equation}

For deriving the residual for the approximate log-Fokker-Planck equation \eqref{u_sde 2}, we use a neural network with parameters $\theta$ trained to approximate the solution $u_\theta$ of \eqref{u_sde 2}. $u_\theta$ satisfies
\begin{equation}\label{new_u_sde 2}
\begin{aligned}
&\frac{\partial u_\theta}{\partial t} + \nabla_x \cdot f_{X,\theta}^{S D E}+ f_{X,\theta}^{S D E} \cdot \nabla_x u_\theta + \nabla_y \cdot f_{Y,\theta}^{S D E} + f_{Y,\theta}^{S D E} \cdot \nabla_y u_\theta+ \frac{1}{2} g_X^2 \left(\Delta_x u_\theta+\left\|\nabla_x u_\theta\right\|^2\right)\\&+\frac{1}{2} g_Y^2 \left(\Delta_y u_\theta+\left\|\nabla_y u_\theta\right\|^2\right)\\
=&\frac{\partial u_\theta}{\partial t}+\nabla_x\cdot(f_X-g_X^2  \nabla_x u_\theta) +(f_X-g_X^2  \nabla_x u_\theta)\cdot \nabla_x u_\theta
+\nabla_y\cdot(f_Y-g_Y^2  \nabla_y u_\theta)\\
&+(f_Y-g_Y^2  \nabla_y u_\theta)\cdot \nabla_y u_\theta+ \frac{1}{2} g_X^2 \left(\Delta_x u_\theta+\left\|\nabla_x u_\theta\right\|^2\right)+\frac{1}{2} g_Y^2 \left(\Delta_y u_\theta+\left\|\nabla_y u_\theta\right\|^2\right)\\
=&\frac{\partial u_\theta}{\partial t}+\nabla_x \cdot f_X + f_X\cdot \nabla_x u_\theta + \nabla_y \cdot f_Y +f_Y\cdot \nabla_y u_\theta-\frac{1}{2} g_X^2\left(\Delta_x u_\theta+\left\|\nabla_x u_\theta\right\|^2\right)\\&-\frac{1}{2} g_Y^2\left(\Delta_y u_\theta+\left\|\nabla_y u_\theta\right\|^2\right).
\end{aligned}
\end{equation}

\paragraph{Residual evaluation}
Physics-informed neural networks (PINNs) approximate PDE solutions by parameterizing the unknown field with a neural network
and minimizing the squared PDE residual evaluated at sampled space--time points.
Here we adopt the same principle for the approximate log-Fokker--Planck equation~\eqref{u_sde 2}:
we represent its solution by a neural potential $u_\theta$ and define the corresponding pointwise residual
\[
r_\theta(x,y,t):=\mathcal{F}_{Y_0}[u_\theta](x,y,t),
\]
where $\mathcal{F}_{Y_0}$ is the corrected log-Fokker--Planck operator associated with the
plug-in modified reverse-time SDE, obtained from \eqref{new_u_sde 2} by replacing the $y$-score
$\nabla_y u$ with $\nabla_y u + c_y(t;Y_0,y)$.

In high dimension, the dominant cost comes from the Laplacian terms
$\Delta_x u_\theta=\mathrm{tr}(\nabla_x^2 u_\theta)$ and $\Delta_y u_\theta=\mathrm{tr}(\nabla_y^2 u_\theta)$.
We therefore compute all first-order derivatives in $\mathcal{F}_{Y_0}[u_\theta]$
(i.e., $\partial_t u_\theta$, $\nabla_x u_\theta$, $\nabla_y u_\theta$) by automatic differentiation,
and when needed we estimate the Laplacians using an unbiased Hutchinson trace estimator implemented via
Hessian--vector products (computed by automatic differentiation).

By the same algebraic manipulation as in the unconditional case,
the corrected residual can be written by adding the plug-in contributions
$-g_Y^2(\nabla_y\!\cdot c_y + c_y\cdot \nabla_y u_\theta)$
to the standard log-Fokker--Planck residual.
Accordingly, we define the (time-averaged, normalised) log-Fokker--Planck residual as
\begin{equation}\label{residual2}
R(\theta,u_\theta,t)
:= V(T-t)^{-1}\int_t^T
\bigl\|\mathcal F_{Y_0}[u_\theta](\cdot,\cdot,s)\bigr\|_{L^2(\Omega_X\times\Omega_Y)}^2\,ds,
\end{equation}
where $V(r):= r\,\mathrm{Vol}(\Omega_X\times\Omega_Y)$.
Equivalently, expanding $\mathcal F_{Y_0}$ yields the integrand in \eqref{residual2}
with the additional plug-in contributions
$-g_Y^2(\nabla_y\!\cdot c_y + c_y\cdot \nabla_y u_\theta)$.
\begin{equation}\label{residual2-expanded}
\begin{aligned}
R\left(\theta, u_\theta, t\right) =& V(T-t)^{-1} \int_t^T \Bigg\| \frac{\partial u_\theta}{\partial s}(\cdot,\cdot, s) + \nabla_x \cdot f_X(\cdot, s) + f_X(\cdot, s) \cdot \nabla_x u_\theta(\cdot, s) \\+ &\nabla_y \cdot f_Y(\cdot, s) + f_Y(\cdot, s) \cdot \nabla_y u_\theta(\cdot,\cdot, s)- \frac{1}{2} g_X^2(s)\left(\Delta_x u_\theta(\cdot,\cdot, s) + \left\|\nabla_x u_\theta(\cdot,\cdot, s)\right\|_2^2\right) \\- &\frac{1}{2} g_Y^2(s)\left(\Delta_y u_\theta(\cdot,\cdot, s) + \left\|\nabla_y u_\theta(\cdot,\cdot, s)\right\|_2^2\right) \\-&g_Y^2(s)\Bigl(\nabla_y\!\cdot c_y(s;Y_0,\cdot)
+ c_y(s;Y_0,\cdot)\cdot \nabla_y u_\theta(\cdot,\cdot,s)\Bigr)\Bigg\|_{L^2(\Omega_X\times \Omega_Y)}^2 ds.
\end{aligned}
\end{equation}
where $V(r):= r\,\mathrm{Vol}(\Omega_X\times\Omega_Y)$.
We refer to $R$ as the (log-)Fokker--Planck residual.

 Or alternatively, we define the approximate drift as
\begin{equation}\label{ode_2}
\begin{cases}
&f_{\theta,X}^{O D E}\left(X_t, t\right)=f_X(X_t, t) - \frac{1}{2} g_X^2(t) \nabla_x \log p_t(X_t, Y_t),\vspace{0.2cm}\\
&f_{\theta,Y}^{O D E}\left(Y_t, t\right)=f_Y(Y_t, t) - \frac{1}{2} g_Y^2(t) \nabla_y \log p_t(X_t,Y_t, t),
\end{cases}
\end{equation}
where  $f_{X,\theta}^{O D E}, f_{Y,\theta}^{O D E} \in C^{\infty}$ and 
$$\left\|f_{X,\theta}^{O D E}(x, t)\right\|_2 \leq\left(K_f+M^2 K_u\right)\left(1+\|x\|_2\right),\quad \left\|f_{Y,\theta}^{O D E}(y, t)\right\|_2 \leq\left(K_f+M^2 K_u\right)\left(1+\|y\|_2\right).$$ 
Using \eqref{ode_2}, we have
\begin{equation}\label{p_ode2}
\begin{cases}
& \displaystyle\frac{d X_t}{d t} = f_{\theta,X}^{O D E}\left(X_t, t\right), \vspace{0.2cm}\\ 
& \displaystyle\frac{d Y_t}{d t} = f_{\theta,Y}^{O D E}\left(Y_t, t\right).
\end{cases}
\end{equation}
In summary, the formulations \eqref{eq:joint-forward}, \eqref{SDEs}, and \eqref{Odes} involve the density \( p \). The approximations \eqref{SDE2} and \eqref{reverse:SDE2}, which are derived from the reverse SDE \eqref{SDEs}, have the density \( p_\theta^{SDE} \). The approximation of the probability flow ODE \eqref{p_ode2} involves the density \( p_\theta^{ODE} \). Additionally, the density ) \(u_\theta=\log p_\theta\)
as given in \eqref{new_u_sde 2}, is directly implied by the neural network’s approximation of the log-density \( u_\theta \). 

It is important to note that \( p \neq p_\theta \neq p_\theta^{SDE} \neq p_\theta^{ODE} \). For most calculations and numerical work, we prefer using the logarithms of densities rather than the densities themselves. We denote the log-density associated with a given density \( p \) as \( u \) or the potential. Specifically, \( u(x, t) = \log p(x, t) \), \( u_\theta^{SDE}(x, t) = \log p_\theta^{SDE}(x, t) \), \( u_\theta^{ODE}(x, t) = \log p_\theta^{ODE}(x, t) \), and \( u_\theta(x, t) = \log p_\theta(x, t) \) for all \( (x, t) \in \Omega \times [0, T] \). Similarly, \( u \neq u_\theta \neq u_\theta^{SDE} \neq u_\theta^{ODE} \).

\begin{theorem}\label{gap2}
Let the residual $R(\theta, u_\theta, t)$ in \eqref{residual2} satisfies $R(\theta, u_\theta, t) < \delta$. 
If $p_\theta^{SDE}$ satisfies \eqref{new_fp sde2|y} and $p_\theta^{ODE}$ satisfies \eqref{p_ode2}, then
$$
W_2\left(p_\theta^{ODE}(\cdot, t), p_\theta^{SDE}(\cdot, t)\right) < C \delta,
$$
where $C > 0$ is a constant independent of $\delta$.
\end{theorem}
\begin{proof}
We proceed in two steps. Firstly, we show that there exists a constant $\tilde{C}>0$ independent of $\delta$ such that $\left\|u_\theta(\cdot, \tau)-u_\theta^{S D E}(\cdot, \tau)\right\|_{L^2(\Omega)}<\tilde{C} \delta$ which allows us to show that $\int_0^\tau\left\|\nabla u_\theta(\cdot, s)-\nabla u_\theta^{S D E}(\cdot, s)\right\|_{L^2(\Omega)}^2 d s<C \delta$ for the reverse time variable $\tau=T-t \in(0, T]$ where the constant $C>0$ is independent of $\delta$. Finally, we prove that
$$
W_2\left(p_\theta^{O D E}(\cdot, t), p_\theta^{S D E}(\cdot, t)\right)<C \delta
$$
for some constant $C>0$ independent of $\delta$.

\textbf{Step I}: To prove the first part, assume that $u_\theta^{S D E}$ solves \eqref{u_sde 2} on $\Omega_X \times \Omega_Y \times[0, T]$, that is
\begin{equation}\label{u_sde 3}
\begin{aligned}
\frac{\partial u_\theta^{SDE}}{\partial t}
+ \nabla_x \cdot f_{X,\theta}^{SDE} +f_{X,\theta}^{SDE} \cdot \nabla_x u_\theta^{SDE}
+ \nabla_y \cdot f_{Y,\theta}^{SDE} +f_{Y,\theta}^{SDE}\cdot \nabla_y u_\theta^{SDE}\\
+\frac{1}{2} g_X^2\left(\Delta_x u_\theta^{SDE}+\left\|\nabla_x u_\theta^{SDE}\right\|^2\right)
+\frac{1}{2} g_Y^2\left(\Delta_y u_\theta^{SDE}+\left\|\nabla_y u_\theta^{SDE}\right\|^2\right)
=0.
\end{aligned}
\end{equation}

equipped with terminal data $u_\theta^{SDE}(\cdot, \cdot,T)$ restricted to $ \Omega_X \times \Omega_Y $ and boundary conditions $u_B^{S D E}=\log p_B^{S D E}$ on $\partial \Omega_X \times \partial \Omega_Y \times[0, T]$. Further, assume that $u_\theta$ satisfies \eqref{u_sde 2} on $\Omega_X \times \Omega_Y \times[0, T]$ with some residual $q$, i.e.
\begin{equation}\label{u_sde 4}
\begin{aligned}
\frac{\partial u_\theta}{\partial t}
+ \nabla_x \cdot f_{X,\theta}^{SDE} +f_{X,\theta}^{SDE} \cdot \nabla_x u_\theta
+ \nabla_y \cdot f_{Y,\theta}^{SDE} +f_{Y,\theta}^{SDE}\cdot \nabla_y u_\theta\\
+\frac{1}{2} g_X^2\left(\Delta_x u_\theta+\left\|\nabla_x u_\theta\right\|^2\right)
+\frac{1}{2} g_Y^2\left(\Delta_y u_\theta+\left\|\nabla_y u_\theta\right\|^2\right)
= q(x,y,t).
\end{aligned}
\end{equation}

equipped with terminal data $u_\theta^{SDE}(\cdot, \cdot,T)$ restricted to $ \Omega_X \times \Omega_Y $ and boundary conditions $u_B^{S D E}=\log p_B^{S D E}$ on $\partial \Omega_X \times \partial \Omega_Y \times[0, T]$.

Since equations \eqref{u_sde 3} and \eqref{u_sde 4} come with terminal conditions, it is more natural to work in reverse time. Introducing the reverse time variable $\tau=T-t$, where $\tau \in[0, T]$ , we transform the forward-time equations into their reverse-time counterparts.
For equation \eqref{u_sde 3}, the reverse-time equation becomes
\begin{equation}\label{ru_sde 3}
\begin{aligned} \frac{\partial \bar{u}_\theta^{S D E}}{\partial \tau}- & \nabla_x \cdot \bar{f}_{X, \theta}^{S D E}-\bar{f}_{X, \theta}^{S D E} \cdot \nabla_x \bar{u}_\theta^{S D E}-\nabla_y \cdot \bar{f}_{Y, \theta}^{S D E}-\bar{f}_{Y, \theta}^{S D E} \cdot \nabla_y \bar{u}_\theta^{S D E} \\ & \quad-\frac{1}{2} \bar{g}_X^2\left(\Delta_x \bar{u}_\theta^{S D E}+\left\|\nabla_x \bar{u}_\theta^{S D E}\right\|^2\right)-\frac{1}{2} \bar{g}_Y^2\left(\Delta_y \bar{u}_\theta^{S D E}+\left\|\nabla_y \bar{u}_\theta^{S D E}\right\|^2\right)=0\end{aligned}\end{equation}
where all terms are now in reverse time, denoted $\tau$.
Similarly, for equation \eqref{u_sde 4}, the reverse-time equation is
\begin{equation}\label{ru_sde 4}
\begin{aligned}
\frac{\partial \bar{u}_\theta}{\partial \tau}- & \nabla_x \cdot \bar{f}_{X, \theta}^{S D E}-\bar{f}_{X, \theta}^{S D E} \cdot \nabla_x \bar{u}_\theta-\nabla_y \cdot \bar{f}_{Y, \theta}^{S D E}-\bar{f}_{Y, \theta}^{S D E} \cdot \nabla_y \bar{u}_\theta \\
& \quad-\frac{1}{2} \bar{g}_X^2\left(\Delta_x \bar{u}_\theta+\left\|\nabla_x \bar{u}_\theta\right\|^2\right)-\frac{1}{2} \bar{g}_Y^2\left(\Delta_y \bar{u}_\theta+\left\|\nabla_y \bar{u}_\theta\right\|^2\right)=\bar{q}(x,y, \tau).
\end{aligned}
\end{equation}
where $\bar{q}(x, \tau)$ is the reverse-time equivalent of $q(x,y,t)$.
Note that the log-Fokker-Planck residual \eqref{residual2} is related to $\bar{q}(\cdot, \tau)$ for $\tau=T-t$ by
$$
R\left(\theta, u_\theta, t\right)=R\left(\theta, u_\theta, T-\tau\right)=V(\tau)^{-1} \int_{T-\tau}^T\|q(\cdot, s)\|_{L^2(\Omega)}^2 d s=V(\tau)^{-1} \int_0^\tau\|\bar{q}(\cdot, s)\|_{L^2(\Omega)}^2 d s
$$
Define the reverse-time error $e_{\bar u}:=\bar u_\theta-\bar u_\theta^{\mathrm{SDE}}$.
In the unified conditional sampler, the plug-in correction enters the reverse-time dynamics only through the
additional $y$-drift term $-g_Y^2(t)c_y(t;Y_0,y)$, which at the PDE level contributes the same extra terms
\[
-\,g_Y^2(t)\Bigl(\nabla_y\!\cdot c_y(t;Y_0,y)
+ c_y(t;Y_0,y)\cdot \nabla_y \bar u(\cdot,\cdot,t)\Bigr),
\qquad \bar u\in\{\bar u_\theta,\bar u_\theta^{\mathrm{SDE}}\}.
\]
to both reverse-time log-Fokker--Planck equations for $\bar u_\theta^{\mathrm{SDE}}$ and $\bar u_\theta$.
Therefore, when subtracting \eqref{ru_sde 3} from \eqref{ru_sde 4} to obtain the error equation for $e_{\bar u}$,
all $c_y$-dependent contributions cancel identically, and the resulting error equation has the same form as in the
unconditional analysis (with $z=(x,y)$ and block-diagonal diffusion). This allows us to proceed with the remainder of the derivation as in the unconditional case, ensuring the method's consistency.

\end{proof}

\section{Additional Comparison with DPS and DiffPIR}
\label{app:additional-comparison}

\subsection{Relation to DPS-style guidance}
\label{app:dps-comparison}

Our formulation and DPS~\cite{chung2022diffusion} can both be motivated
from Bayes' rule, but they induce different conditional sampling mechanisms.
DPS corrects the marginal $x$-score using an approximate likelihood gradient
$\nabla_x \log p_t(y\mid x_t)$, evaluated through per-step backpropagation
through the denoising network. In contrast, our formulation decomposes the
joint conditional score and introduces an analytic correction in the
$y$-channel, while its effect on the target is mediated through the learned
joint dependence $s_x(X_t,Y_t,t)$.

This construction yields explicit conditional reverse-time SDE and
probability-flow ODE dynamics. In the pretrained-prior setting, the analytic
plug-in correction can also be evaluated without per-step backpropagation
through the score network. At equal NFE, this leads to substantially lower
sampling cost than DPS, as quantified below.

\subsection{Runtime and distributional evaluation}
\label{app:runtime-fid}

We additionally compare computational cost and distributional quality under
the same sampling budget. Table~\ref{tab:runtime} reports results on
ImageNet256 noisy inpainting using the same 100 images and
$\mathrm{NFE}=1000$ as in the corresponding main-paper experiment.

\begin{table}[t]
\caption{Runtime and distributional evaluation on ImageNet256 noisy inpainting
using one RTX~3090 with batch size 1. All methods use the same 100 images and
sampling budget (NFE = 1000).}
\label{tab:runtime}
\centering
\footnotesize
\setlength{\tabcolsep}{3.5pt}
\begin{tabular}{lcccccc}
\toprule
Method & NFE & Backprop & s/img & FID$\downarrow$ & SSIM$\uparrow$ & JFID$\downarrow$ \\
\midrule
DPS     & 1000 & yes & 245.1 & 84.47 & 0.694 & 44.10 \\
DiffPIR & 1000 & no  & 128.0 & \textbf{38.22} & 0.753 & 17.35 \\
Ours    & 1000 & no  & 130.8 & 39.64 & \textbf{0.764} & \textbf{16.73} \\
\bottomrule
\end{tabular}
\end{table}

At equal NFE, our method requires $130.8$ seconds per image, compared with
$245.1$ seconds for DPS, corresponding to approximately $1.9\times$ faster
sampling. Its runtime is close to DiffPIR ($128.0$ seconds per image),
indicating that the analytic plug-in correction introduces little additional
computational overhead. Our method achieves an FID of $39.64$, close to
DiffPIR's $38.22$ and substantially lower than DPS's $84.47$, while obtaining
the best SSIM and JFID among the compared methods.


\end{document}